\documentclass[journal]{IEEEtran}
\usepackage{multirow} 
\usepackage[style=ieee]{biblatex} 
\bibliography{mybibliography.bib}
\usepackage{hyperref}
\usepackage{amsmath}
 \usepackage{multirow}
 \usepackage{siunitx}
 \usepackage{soul}
 \usepackage{graphicx} 
\usepackage{amsmath,amssymb,amsthm,epsfig}
\newtheorem{theorem}{Theorem}[section]
\usepackage{hyperref}
\usepackage{xcolor}
\usepackage[font=small,labelfont=bf]{caption}
\usepackage{url}
\usepackage{graphicx}  
\usepackage{float} 
\usepackage{booktabs} 
\usepackage{array}
\newcolumntype{R}{>{\raggedleft\arraybackslash}p{1.5em}}
\usepackage{graphicx} 
\usepackage{bm}
\usepackage{tabularx}
\usepackage{multirow}
\usepackage{booktabs}
\usepackage{makecell}

\begin{document}
\title{{Statistical Analysis of using the Shapley Value for Sensor Anomaly Localization with Accurate Classifiers}\thanks{ This work was supported by the U.S. Office of Naval Research under Grant N00014-22-1-2626.}}  
\author{Xubin Fang, Rick S. Blum, \IEEEmembership{Life Fellow, IEEE}  \thanks{Xubin Fang and Rick S. Blum are with the Electrical and Computer Engineering  Department of Lehigh University (emails: xuf220@lehigh.edu, rblum@eecs.lehigh.edu). }}
\maketitle

%{\color{red} Think carefully about if we should rewrite this since this may be a bad way to introduce our ideas. This does not explain why we consider $P_e$. Connecting localization and classification is not trivial.}
\begin{abstract}
%\textcolor{red}{In the abstract, comparison of pe in Shapley takes some explaining, address it in the new paper like the old paper}

Recent publications have suggested using the Shapley value for sensor anomaly/attack localization. We study the performance of such an approach by using mathematically defined optimum binary classifiers in the Shapley value calculation. 
To judge localization performance, we study the ability of the Shapley value 
of a given sensor observation to determine if that observation is anomalous. First, we prove that for cases with independent sensor observations, an  optimized anomaly test using the Shapley value is equivalent to an optimized lower-complexity anomaly test using a single term in the Shapley value calculation, yielding the exact same probability of error.

For some popular dependent observation cases involving two sensors, including correlated bivariate Gaussian/Laplacian probability density functions and constant/Gaussian attacks/anomalies, we prove that these two tests are fundamentally different, yielding different decision regions and error probabilities. Further, we prove  that the Shapley value test is sometimes strictly inferior to the other (single term in Shapley calculation) test in certain statistically dependent bivariate Gaussian scenarios with large correlation magnitude and 
additive attacks/anomalies, while it is strictly superior in others, depending on the sign of the correlation. One can combine these two approaches to obtain a strictly better approach in these cases. 
These results, which provide the first theoretical statistical analysis of Shapley-based localization, seem very interesting based on 
the wide acceptance of the Shapley value by many researchers and should encourage further research on this topic.  Numerical results are provided which illustrate our findings. 
\end{abstract}

\begin{IEEEkeywords}
Shapley value, anomaly detection, anomaly localization, feature attribution 
\end{IEEEkeywords}

\section{Introduction}\label{intro}

Sensors are frequently employed to monitor the operation of many important real world systems \cite{willettnew,Varsheynew,BChennew}.  On the other hand, cyber attacks/anomalies on these sensor 
systems are well known to cause significant issues and 
to date these systems need further 
protection against these attacks \cite{sensors-security}.   The popularity of the internet of things, where sensor systems are connected to the internet 
\cite{GRR2020}, makes it even more important to provide security for these systems since the internet connections can provide easier access to the sensors. 

The recent work in \cite{E-sfd,ameli2022unsupervised}  
described a new approach which could be very helpful 
in identifying cyber attacked/anomalous sensor data.  The approach employs the popular Shapley value to localize attacked/anomalous sensor data, thus describing the sensors whose data is anomalous/attacked.  The Shapley value has been very popular in the game theory and machine learning communities \cite{li2023surveyexplainableanomalydetection}, which makes this new idea worthy of additional study. 
We theoretically investigate this approach here, in a mathematical and well controlled setting. We focus on cases where optimum classifiers are used in the Shapley value calculation.  
Such results should be very meaningful in practice, where people attempt to employ classifiers which achieve performance which is close to optimum. 

{ 
While the Shapley value has received very little statistical analysis to date for sensor anomaly localization, the proper place for such analysis to appear and be peer reviewed 
%in the IEEE community 
is the IEEE Transactions on Signal Processing, which has published rigorous statistical analysis of similar decision making/localization algorithms employed for many different applications for its long history as evidenced by \cite{aalo1992asymptotic, 
RNiu, chen2000optimality, chamberland2003decentralized, marano2009distributed, 
visa1,
stoica1989music, viberg1991sensor, friedlander1990sensitivity, sayed1997error, horowitz1981performance, stoica2004model, swami1991identifiability, ben2010cramer, lahat2012second, picinbono1995likelihood}. 
It should be noted that the analysis presented here uses similar statistical techniques as those that are frequently employed in Statistical Signal Processing, often for sensor observations. 
}

Let $N$ sensors each provide  a scalar observation to obtain a  whole set of observations denoted by $ x_1, x_2, ..., x_N$. 
The formula for calculating  the Shapley value %\textcolor{red} 
{associated} with observation $x_i, 1 \leq i \leq N$, is (see 
\cite{watson2023explainingpredictiveuncertaintyinformation} 
for a detailed game theory derivation)   
\begin{equation}\label{shap}
    \phi (x_i)= \sum_{S \subseteq {\cal N} / (i)} \frac{ |S| ! (N- |S| -1)!}{N!} (v (S \cup (i)) - v(S))
\end{equation} 
where $\phi(x_i)$ denotes the Shapley value for the $i$th sensor observation $x_i$, ${\cal N} = \{ 1, 2, ..., N\}$ is a set which denotes all possible sensor indices, $v$ denotes a soft classifier usually derived from machine learning (here defined mathematically for repeatability/control), and $S$ denotes a particular subset of ${\cal N}$.  Note that each instance of $v$ (standard notation) in (\ref{shap}) generally employs a different number of arguments (sensor indices). 
The output of the soft classifier $v$ describes the likelihood of an attack/anomaly being present in the 
sensor data corresponding to the sensor indices input to $v$. 
Thus, $v$ should produce  a large positive %\textcolor{red}
{scalar} output value if any of the sensor data corresponding to the indices input 
to it are attacked/anomalous.  On the other hand, 
$v$ should produce a very  negative output when none of the sensor data corresponding to the indices input to it are attacked/anomalous.  

The sum in (\ref{shap}) is over all subsets $S $ of  sensors with indices chosen from ${\cal N}$, excluding sensor $i$. A term in the sum in (\ref{shap}) can be seen to 
be composed of the product of two quantities. 
The first quantity, $\frac{ |S| ! (N- |S| -1)!}{N!}$,  
depends on the cardinality of the set $S$, denoted by  $|S|$. This term is $1/N$ (to average) times one over the number of ways to choose $|S|$ things from $N-1$ things (recall we omitted the $i$th sensor). 
 The second quantity in the product, $v (S \cup (i)) - v(S)$ in (\ref{shap}), is the difference between the classifier output 
 with a selected set 
 of sensors including sensor $i$ and the classifier output with the same sensor inputs except the data from sensor $i$ is excluded. This difference 
 expresses the value of the data from sensor $i$. 

% stopped 

To obtain results where we can provide the desired mathematical proofs, 
we set the accurate classifier function $v(\cdot)$ to be the optimum classifier 
(called the Bayes-optimum classifier) 
\cite{poor2013introduction} for the given sensor data provided to $v(\cdot)$, 
the log-likelihood Ratio classifier for the given classification problem.  This allows us to provide a repeatable and controllable analysis needed for mathematical proofs for the mathematically described classification problems we consider.

 {From (\ref{shap}), one term in the sum is $ \mbox{constant} \times v(i)$.  Thus $ \phi(x_i) = \mbox{constant} \times v(i) + \mbox{ other terms} $. }
 Since the Shapley value uses the quantity $v(i)$ 
as part of %\textcolor{red}
{its} calculation, then 
if the Shapley value 
is good for localizing attacks/anomalies one 
would think it would generally provide better performance than that obtained when using $v(i)$ 
directly for this task. However, the 
results in this paper do not indicate this is true in the cases we considered. 
To demonstrate this, we explicitly compare the probability of errors of two tests which decide 
if a given attack/anomaly 
includes sensor $i$. Each test will compare a function to a threshold. If the function is larger than the threshold, then the attack/anomaly is decided 
to include sensor $ i$. If the function is smaller than the threshold, then 
the attack/anomaly is decided 
to not include sensor $ i$.
One test compares $ \phi(x_i)$ to an optimized threshold.  The other test compares $ v(i)$ to an optimized threshold. 
The threshold of each test is chosen to minimize the probability of error of this test. 
The test which gives smaller probability of error is clearly better for localizing the attack/anomaly at sensor $i$. 

Since $v(i)$ is one of many terms in $ \phi(x_i)$, it is not surprising that 
$ \phi(x_i)$ requires additional computations, when compared to $v(i)$. 
To quantify the additional computations, we can consider 
the time complexity, which measures the number of operations 
which must be 
performed when the 
input size, $N$ here,  
is specified. 
To obtain $ v(i)$ in practice (where we do not have a mathematical description of  $ v(i)$), 
we learn a single function and evaluate it for the given sensor data.  To compute $ \phi(x_i)$ in practice, we 
need to learn roughly $ \sum_{j=0}^{N} \binom{N}{j} $ functions, as per (\ref{shap}), and perform %$O(N 2^N)$ 
$O(2^N)$ computations with the evaluations of the learned  functions at the given sensor data.  Even ignoring the learning part, it always requires more effort/time to compute $ \phi(x_i)$ and the complexity grows exponentially in $N$.

{
Surprisingly, using the mathematical model of $v$ just described, our numerical results show that comparing $ v(i)$ to an optimized threshold performs equivalently in terms of probability of error to the test using $ \phi(i)$ for all the cases with independent observations we have considered. We give an analytical proof showing this must be true for all independent observation cases. The proof shows the two tests are exactly the same. 
Thus, for independent observation cases, using $ v(i)$ performs as well as using $ \phi(x_i)$, with lower complexity. 

For scenarios involving dependent observations from two sensors, where only the first sensor is subject to attack/anomaly  (so we consider only $v(1)$ and $\phi(x_1)$), we numerically demonstrated that 
the two tests are different for all cases considered which included some constant additive and Gaussian %\textcolor{red}
{attacks/anomalies}. Further, we are able to prove that the two tests must be different (different probabilities of error and decision regions) for some popular bivariate 
probability density functions (pdfs) under constant additive attacks/anomalies.  We were able to prove this is also true for bivariate Gaussian data under a Gaussian attack/anomaly. We also provide numerical results that show that for high correlation magnitude  %\textcolor{red}
{bivariate} Gaussian data under a constant additive attack/anomaly, sometimes the test using $v(1)$ is better and other times the Shapley test is better. In particular,  Shapley is better in these cases if $\rho>0$, while $v(1)$ is better if  $\rho<0$.  We were able to prove this will always be the case for nonzero additive attacks/anomalies if $|\rho|$ is sufficiently close to one for the considered cases. 

\section{Literature Review}
{
\subsection{Shapley Value: Game Theory, Explainable AI, and Anomaly Detection}}
\par The Shapley value 
%itself 
\cite{Shapley1953} came out of game theory. It {assigns a possibly different value   to each player which represents the determined worth of their contribution to the results from the  group of players. This could be used to  distribute the} total payment to the group.
%or cost. 
The complexity of the Shapley value calculation increases rapidly as the number of players increases and so does the time needed to compute it. 
This makes it very challenging to perform the calculation for large systems. Recently, the Shapley value has been used to explain results from machine learning algorithms, as can be seen in many papers, see for example the review papers  \cite{DeepLearningADAreview, DeppLearningADAsurvey} for more discussion. 
In \cite{li2023surveyexplainableanomalydetection}, an overview is provided of how the Shapley value  and other alternative methods are used in  explainable anomaly detection.  On the other hand, there are many papers related to anomaly detection that do not specifically consider the Shapley value, see the references in  \cite{DeepLearningADAreview, DeppLearningADAsurvey} for example.

\subsection{{Shapley Value in Sensor Anomaly Localization}}
We previously mentioned that \cite{E-sfd,ameli2022unsupervised} suggested using the Shapley value in sensor anomaly localization. 
In \cite{E-sfd}, the authors employ a simplified version of the Shapley value to pinpoint the sensors  at fault in an industrial control system  application. 
In  \cite{ameli2022unsupervised}, the authors also suggest using the Shapley value 
for sensor anomaly localization. 
%, but test these ideas using a non-sensor server machine data set.  

{\subsection{Shapley Value in 
Feature Localization}}
Other research attempts to localize which inputs to a machine learning algorithm most impact 
a particular output decision. We call this feature localization. These studies may or may not be related to sensors or anomaly detection.  
In \cite{XAIShapleyincomplexADML}, the Shapley value and simplifications of the Shapley value are used for feature localization in an anomaly detection application. 
In \cite{networkpacket}, a simplification of the Shapley value is utilized in network traffic data to identify which features are most important for some particular decisions. 
In \cite{characteristicfunction}, the Shapley value is used in tandem with a characteristic function for post-hoc feature localization. The algorithm is tested on different kinds of medical data, some of which may come from sensors. 

In \cite{PCA}, the Shapley value is used to localize reconstruction errors from a principal component analysis.   This is tested on various datasets ranging from {cardiovascular} data, forest cover, radar returns, mammography and satellite imaging. 
The research in \cite{autoencoders} 
applied a simplification of the Shapley value for feature localization 
in autoencoder networks employed for 
anomaly detection.  Various datasets were used in the testing, including warranty claim datasets, credit card fraud detection, military network intrusion detection, and an artificial dataset.  
The research in \cite{modelindependentfeatureattribution} uses a Shapley value-based method for feature localization.   The approach is tested on artificial datasets and medical data.  
The research in \cite{explainingindividual} also employs a simplification of the  Shapley value 
for feature localization, while being tested on
simulated and real mortgage default data. 

{\subsection{Theoretical Analysis of the Shapley Value for Machine Learning}}
The authors in \cite{owen2017shapleyvaluemeasuringimportanceanova} study feature localization by showing that it gives similar 
results as an analysis of variance method. 
In \cite{sundararajan2020many}, the authors compare different Shapley methods 
 theoretically and mathematically to  highlight their advantages for 
 different machine learning 
 models and applications. 
Most importantly, we have not seen any papers in the literature that study the issues enumerated in the last two paragraphs of the Section~\ref{intro}, thus justifying the novelty of this paper. 

{ In a previous paper \cite{blum2025usingshapleyvalueanomaly}, we compared the Shapley value and $v(i)$ tests for a different suboptimal, but reasonable,  mathematical model for the $v(\cdot)$ functions in the Shapley value calculation.  We were able to prove that both tests are the same for cases with statistically independent observations. We did not provide any theorems for the case of statistically dependent observations.  Instead we gave a single numerical result which showed some bad behavior of the Shapley value test for two sensor cases with a statistically dependent bivariate Gaussian distribution when we try to determine if the first sensor is attacked/anomalous when only the second sensor is actually attacked/anomalous.   

In this paper we consider using the optimum classifiers for the $v(\cdot)$ functions in the Shapley value calculation.
Using these $v(\cdot)$ functions, different from those used in \cite{blum2025usingshapleyvalueanomaly}, we show, for the first time,  
the Shapley value and $v(i)$ tests are identical for independent observation cases.  For some popular dependent observation pdfs with two sensors,  including correlated bivariate Gaussian cases and some nonGaussian cases which generalize the bivariate Gaussian case,
we first show 
the Shapley value and $v(i)$ tests must be different and they must produce different probabilities of error. 

Further, we 
prove the Shapley test is strictly inferior to the test using $v(i)$ in this %\textcolor{red}
{paper} for some  statistically dependent bivariate Gaussian observation cases with highly negative correlation, while we prove the Shapley value is better in some other cases with highly positive correlation.  It may be surprising that the Shapley test is not always better, considering the popularity of the Shapley value.  Numerical results agree with these findings.}

\section{Analytical Results}
\label{analytical}
% {\color{red} Check all of this one more time. }\\
We first consider the simpler case of independent observations. 
\subsection{Independent Sensors}
We now show that the test using the Shapley value and the test using $v(i)$ are  the same test for independent sensor observation cases when the optimum classifiers are used for the $v(\cdot)$ functions in the Shapley calculation. 
We make the following assumptions for the following Theorem (Theorem III.1). \begin{enumerate}
\item Assume the unattacked/nonanomalous sensor data at a given time $x_1,x_2,...,x_N$ are statistically independent under $C_0$ and $C_1$, each $x_i, i = 1,\ldots,N $  following the marginal pdf or probability mass function (pmf) $f(x_i|C_j)$ for $i = 1,\ldots,N $ and $j=0,1$.
\item   As discussed previously, we define 
the classifiers $v(\cdot)$ used in the Shapley calculation as the optimum classifiers, natural log of the likelihood ratio  of the sensor data corresponding to the indices in the input to $v$.  The likelihood ratio is the pdf/pmf of the observed data under $C_1$ divided by the pdf/pmf of of the observed data under $C_0$,
assuming the denominator is not zero. Thus 
in this case, 
if the observed sensor data vector is ${\bf x}$, then the likelihood ratio\footnote{The likelihood ratio is more carefully defined in Theorem~\ref{Them-LRTOptimalityGeneral} in order to avoid division by zero.} is $\frac{f({\bf x}|C_1)}{f({\bf x}|C_0) } $. 
This assumption 
(defining the $v(\cdot)$ functions used in the Shapely calculation)  
holds everywhere in this %\textcolor{red}
{paper} (all theorems and numerical results) even if the data are statistically dependent, a case considered outside of the following Theorem. 
\end{enumerate}

\begin{theorem}
\label{Th1}
Under the previous assumptions 1 and 2, a test based on comparing the Shapley value $\phi(x_i)$ to an optimized threshold is exactly the same 
(also same error probability) 
as a test based on comparing $v(i)$ to an optimized threshold. In both cases, the threshold is optimized to minimize the probability of error for the given test. 
\end{theorem}
\begin{proof}
%\label{Th1}
Recall the equation for the Shapley value  of the ith sensor data is  
\begin{eqnarray}
\label{shap-eq3}
    \phi (x_i) = \sum_{S \subseteq {\cal N} / (i)} \frac{ |S|! (n - |S| - 1)!}{n!} \left( v(S \cup (i)) - v(S) \right).
\end{eqnarray}
As assumed (for sensor observations  $x_1,\ldots,x_{N}$) 
\begin{eqnarray}
    v(x_1,x_2,\ldots,x_N) = ln\left(\frac{f(x_1,x_2,\ldots,x_N|C_1)}{f(x_1,x_2,\ldots,x_N|C_0)}\right).
\end{eqnarray}
First consider  $i=N$ and $S=x_1,\ldots,x_{N-1}$ where 
%($ln(abc) = ln(a) + ln(b) + ln(c) $) 
\begin{eqnarray}
  v(S \cup (i)) - v(S) &=& \ln{\left( \frac{f(x_1, x_2, \ldots, x_N | C_1)}{f(x_1, x_2, \ldots, x_N | C_0)} \right)} \nonumber \\
  &-& \ln{\left( \frac{f(x_1, x_2, \ldots, x_{N-1} | C_1)}{f(x_1, x_2, \ldots, x_{N-1} | C_0)} \right)}.
\end{eqnarray}
We can factorize the joint distribution under either class to obtain 
%\begin{eqnarray}f(x_1, \ldots, x_N | C_k) = f(x_1, \ldots, x_{N-1} | C_k) \cdot f(x_N | C_k).\end{eqnarray}
%Thus, 
%$ v(S \cup (i))$ becomes $f(x_1, x_2, \ldots, x_N | C_k)$.
\begin{eqnarray}
  v(S \cup (i)) = \ln \left( \frac{f(x_1, x_2, \ldots, x_{N-1} | C_1) \cdot f(x_N | C_1)}{f(x_1, x_2, \ldots, x_{N-1} | C_0) \cdot f(x_N | C_0)} \right).
\end{eqnarray}
Using $\ln(abc) = \ln(a)+\ln(b)+\ln(c)$ yields 
%, thus resulting in the following equation. 
\begin{eqnarray}
%\begin{split}
  v(S \cup (i)) - v(S) &= \ln \left( \frac{f(x_1, x_2, \ldots, x_{N-1} | C_1)}{f(x_1, x_2, \ldots, x_{N-1} | C_0)} \right) \nonumber \\
  &\quad + \ln \left( \frac{f(x_N | C_1)}{f(x_N | C_0)} \right) \nonumber \\
  &\quad - \ln{\left( \frac{f(x_1, x_2, \ldots, x_{N-1} | C_1)}{f(x_1, x_2, \ldots, x_{N-1} | C_0)} \right)} \nonumber \\ 
%\end{split}
%\end{equation}
%The first and last term cancel out, leaving the following term for $v(S \cup (i)) - v(S)$. 
%\begin{equation}
%  v(S \cup (i)) - v(S) 
&=  \left( \frac{f(x_N | C_1)}{f(x_N | C_0)} \right)
=  \left( \frac{f(x_i | C_1)}{f(x_i | C_0)} \right) = v(i). 
\end{eqnarray}
In fact, we get exactly the same result for any $S$ and for any $1 \leq i \leq N$ 
so that (\ref{shap-eq3}) yields 
\begin{eqnarray}
    \phi(x_i) = {v(i)} \sum_{S \subseteq {\cal N} / (i)} \frac{ |S|! (n - |S| - 1)!}{n!},
\end{eqnarray}
for $ 1 \leq i \leq N$. 
Thus, denoting  $C =  \sum_{S \subseteq {\cal N} / (i)} \frac{ |S|! (n - |S| - 1)!}{n!} > 0$, 
a constant in $x_i$, we obtain 
%the Shapley value for a given data point $x_N$ is equivalent to the value function times the factor $K$. 
\begin{eqnarray}
    \phi(x_i) = C {v(i)}, \quad   \forall 1 \leq i \leq N. 
\end{eqnarray}
Thus comparing $\phi(x_i)$ to a optimum threshold $\tau > 0 $ is the same as comparing $ v(i)$ to $\tau / C$, which must be an optimum threshold for the test statistic $ v(i)$. 
Thus, the tests are equivalent.
%Similar conclusions hold for evaluation of any valid $S$ and $i$. 
\end{proof}

\subsection{Dependent Sensors}

We first give a theorem 
to explain in what sense the log-likelihood classifiers can be considered optimum.
%\textcolor{red}
{To be clear}, the log-likelihood ratio classifiers are optimum when we test whether all 
sensors in the 
observation vector are attacked/anomalous.  The theorem also shows the suboptimality of classifiers which make different decisions 
from  the log-likelihood ratio classifiers 
for sets that occur with nonzero probability.  
\begin{theorem}\label{Them-LRTOptimalityGeneral} Consider a classification problem
between class $C_1$ (the class where sensor observation vector ${\bf x}$ is attacked/anomalous) and $C_0$
(the class where the sensor observation vector ${\bf x}$ is not attacked/anomalous). Assume the 
pdf/pmf $ f({\bf x}|C_1) $ 
of ${\bf x}$ under class $C_1$
is known and the pdf/pmf $ f({\bf x}|C_0) $ 
of ${\bf x}$ under class $C_0$
is also known. 
Define the prior probabilities as 
$\pi_0 = Pr(C_0)$ and $ \pi_1  = Pr(C_1)$. 
Define the probability of 
error as $p_e = \pi_1 Pr(C_0|C_{1}) + \pi_0 Pr(C_{1}|C_{0})$ where 
$ Pr(C_0|C_{1}) $ denotes the probability the classifier chooses $C_0$ when $C_1$ is true and 
$Pr(C_{1}|C_{0})$ is defined similarly. For any set of $ {\bf x} $ 
occurring with nonzero probability under either $C_0$ and $C_1$, 
the classifier rule that minimizes the probability of 
error should decide for class $C_1$ (the sensor observation ${\bf x}$ is attacked/anomalous)
for all ${\bf x}$ in the set 
$\Gamma_1^{opt} = \{ {\bf x} : \pi_1
f({\bf x}|C_1) > \pi_0 f({\bf x}|C_0)  \}
$ and should decide for 
class $C_0$ otherwise, unless 
$ \pi_1
f({\bf x}|C_1) = \pi_0 f({\bf x}|C_0) $ 
in which case either decision gives the same performance. 
Further, if a classifier 
makes a different decision 
for some set of ${\bf x}$ which occurs with nonzero probability under either class $C_1$ and $C_0$  then this classifier will not be able to achieve the optimum probability of error and will achieve a larger probability of error.  For any  set of $ {\bf x} $
that occurs with zero probability under both class $C_1$ and class $C_0$,  the decision for that set of $ {\bf x} $ will not impact performance. 
The classifier described here is sometimes called a likelihood ratio classifier or a log-likelihood classifier. The log-likelihood classifier  uses an equivalent classifier that takes of the log of both sides of the comparison in $\Gamma_1^{opt}$, which changes nothing. 
\end{theorem}

\begin{proof} 
Let $ \Gamma_1 $ be the set $ {\bf x}$ where the classifier decides for  $C_1$ and $ \Gamma_0 $ be the set $ {\bf x}$ where the classifier decides for $C_0$. 
The probability of error $p_e$ can be expressed as
\begin{equation}
p_e = \pi_0 \int_{\Gamma_1} f({\bf x}|C_0) d{\bf x} + \pi_1 \int_{\Gamma_0} f({\bf x}|C_1) d{\bf x}.
\end{equation}
By using  $\int_{\Gamma_0} f({\bf x}|C_i) d{\bf x} = 1 - \int_{\Gamma_1} f({\bf x}|C_i) d{\bf x}$, the probability of error becomes 
\begin{equation}
p_e = \pi_1 + \int_{\Gamma_1} \left[ \pi_0 f({\bf x}|C_0) - \pi_1 f({\bf x}|C_1) \right] d{\bf x}.
\label{eq2} 
\end{equation}
To minimize $p_e$, the decision region $\Gamma_1$ must include all regions of ${\bf x}$ where the integrand is negative and exclude all regions where it is positive, assuming these regions of $ {\bf x} $ occur with nonzero probability under either $C_0$ or $C_1$. From 
(\ref{eq2}), 
picking $\Gamma_1$ this way will include all the possible regions of ${\bf x}$ in the integral  that will decrease $p_e$ and excludes 
all the possible regions of ${\bf x}$  that will increase $p_e$.   
Thus, the minimum probability of error is achieved if and only if
\begin{equation}
\Gamma_1 = \Gamma_1^{opt} = \{ {\bf x} : \pi_1 f({\bf x}|C_1) > \pi_0 f({\bf x}|C_0) \}
\end{equation}
and all other $ {\bf x} $ should be assigned to $\Gamma_0$, 
except for $ {\bf x} $ that occur with zero probability under both classes and $ {\bf x} $ for which 
$ \pi_1 f({\bf x}|C_1) = \pi_0 f({\bf x}|C_0) $. For these sets of $ {\bf x} $, we can make either decision with no impact on performance from 
(\ref{eq2}). 
This confirms that the optimal test is always a likelihood ratio test.

Similarly, if any region of $ {\bf x} $, which occurs with nonzero probability under either $C_0$ or $C_1$, is assigned differently, (\ref{eq2}) and the discussion after it explains why the  probability of error will be larger that the minimum possible value. 
%Note that the log-likelihood ratio classifier, which takes the log of both sides of the likelihood ratio classifier, is equivalent to the likelihood ratio classifier. 
\end{proof}

It is important to note that, as considered later, if we want to know 
if only one sensor in a larger vector of  observed statistically dependent sensors is  attacked/anomalous and besides the data from the one sensor under test we are given other observed data statistically dependent with it, then Theorem~\ref{Them-LRTOptimalityGeneral} does not describe the optimum test.  So we will investigate this question further later.  First we concern ourselves with whether the Shapley test using statistically dependent $x_1,x_2$ and the test using $v(1)$ can ever be the same for some common statistically dependent observation cases.  We next show some cases where this is not possible, in opposition to 
Theorem~\ref{Th1}. 

\subsubsection{Shapley and $v(1)$ Tests are Different for some pdfs}

%{\color{red} You need to move the bivariate pdf (35) here to introduce it.
In the next theorem, we consider the most popular bivariate pdf, the bivariate Gaussian pdf  
\begin{equation}
{f(x_1, x_2) = \frac{1}{2\pi\sigma_1\sigma_2\sqrt{1-\rho^2}} \exp\left( -\frac{\zeta}{2(1-\rho^2)} \right)},
\label{BG}
\end{equation}
{where $\zeta = (\frac{x_1-\mu_1}{\sigma_1})^2 + (\frac{x_2-\mu_2}{\sigma_2})^2 - 2\rho(\frac{x_1-\mu_1}{\sigma_1})(\frac{x_2-\mu_2}{\sigma_2})$,}
the mean vector is $(x_1,x_2)^T$, the standard deviation vector is $(\sigma_1,\sigma_2)^T$ and $-1 < \rho < 1 $ is the correlation coefficient.  The following theorem specifically shows that  the Shapley test is different from the test using $v(1)$ for a constant nonzero attack/anomaly at sensor 1 when $\rho \neq 0$ in (\ref{BG}), which corresponds to statistically dependent observation cases for the bivariate Gaussian pdf.  

\begin{theorem} \label{Them-ShapleySuboptimalityGaussian} Let the two sensor unattacked/nonanomalous sensor data follow a bivariate Gaussian distribution under class $C_0$ and assume a constant additive nonzero attack/anomaly $A$, which we want to detect the presence of, is added to the first sensor observation under class $C_1$. 

Let 
$[x_1, x_2]^T$ denote the possibly attacked/anomalous sensor observations to which we will apply the classifier. Let the classifier decide if there is an attack/anomaly at sensor 1 by employing a log-likelihood ratio classifier at sensor 1 so that the test statistic compares 
\begin{equation} 
v(1) = \ln \frac{f(x_1 | C_1)}{f(x_1 | C_0)} = \frac{A}{\sigma_1^2}(x_1 - \mu_1) - \frac{A^2}{2\sigma_1^2}. \label{T3.3v1}\end{equation}
with an optimum threshold. 
Then the test using the Shapley function must be different from the test using $v(1)$ if both $A$ and $\rho$ are nonzero\footnote{We always assume $ 0 < \sigma_1, \sigma_2 < \infty $.}.  
\end{theorem}

\begin{proof} Let $\Gamma_1^{\phi} = \{ [x_1, x_2]^T : \phi(x_1) > \tau^s \}$ and 
$\Gamma_1^{v} = \{ [x_1, x_2]^T : v(1) > \tau^v \}$
where $\tau^s$ and $\tau^v$ are the optimum thresholds (minimum error probability) for the two  tests. The boundary of $\Gamma_1^{v}$ is defined by $v(1) = \tau^v$. So, from (\ref{T3.3v1}), it is the line of all $[x_1,x_2]$ pairs satisfying  
\begin{equation} 
x_1 = \mu_1 + \frac{\sigma_1^2}{A}(\tau^v + \frac{A^2}{2\sigma_1^2}) \triangleq \kappa,  \mbox{ and }
-\infty \leq x_2 \leq \infty.
\label{t3.3G1}\end{equation}   
{Since $x_2$ is not attacked/anomalous, we find 
%$v_2$ and $v(1,2)$  denoted as
\begin{equation}
v(2) = \ln \frac{f(x_2 | C_1)}{f(x_2 | C_0)} = 0 \label{v2}\end{equation}
(since $ {f(x_2 | C_1)}={f(x_2 | C_0)} $ due to no attack/anomaly at $x_2$) and
\begin{equation}
\begin{aligned}
 v(1,2) &= \ln \frac{f(x_1, x_2 | C_1)}{f(x_1, x_2 | C_0)} \\
 &= \frac{A}{(1-\rho^2)\sigma_1^2}(x_1-\mu_1) \\
 &\quad - \frac{A\rho}{(1-\rho^2)\sigma_1\sigma_2}(x_2-\mu_2) - \frac{A^2}{2(1-\rho^2)\sigma_1^2}. 
\end{aligned}
\label{v2}
\end{equation}}
Using (\ref{shap}) for the two dimensional case considered yields 
\begin{equation} 
\phi(x_1) = (v(1) + v(1,2) - v(2))/2,  
\end{equation} 
so that 
\begin{equation}
\begin{aligned}
\phi(x_1) = & \frac{A(2-\rho^2)}{2(1-\rho^2)\sigma_1^2}(x_1-\mu_1) - \frac{A\rho}{2(1-\rho^2)\sigma_1\sigma_2}(x_2-\mu_2) \\
& - \frac{A^2(2-\rho^2)}{4(1-\rho^2)\sigma_1^2}.
\end{aligned}
\label{eq18} 
\end{equation}
The boundary of 
$\Gamma_1^{\phi}$  is defined by $\phi(x_1) = \tau^s$, which becomes 
%.  Rearranging the expression for $\phi(x_1)$, we obtain the boundary equation 
\begin{equation}
\begin{aligned}
& \frac{A(2-\rho^2)}{2(1-\rho^2)\sigma_1^2}(x_1-\mu_1) - \frac{A\rho}{2(1-\rho^2)\sigma_1\sigma_2}(x_2-\mu_2) \\
& = \tau^s + \frac{A^2(2-\rho^2)}{4(1-\rho^2)\sigma_1^2}.
\end{aligned}
\label{t3.3Gphi} 
\end{equation}
This is a linear equation of the form $a x_1 + b x_2 = c$. For the regions $\Gamma_1^{\phi}$ 
and $\Gamma_1^{v}$ to be identical, the boundaries must coincide for all $x_1, x_2$. This requires that (\ref{t3.3Gphi}) should not 
depend on $x_2$, like  (\ref{t3.3G1}),  
so the coefficient multiplying $x_2$, $\frac{A\rho}{2(1-\rho^2)\sigma_1\sigma_2}$, must be zero, which is impossible under our stated assumptions 
$A,\rho \neq 0$.  Thus, $\Gamma_1^{\phi} \neq \Gamma_1^{v}$ for all cases where the sensor data are correlated ($\rho \neq 0$). %From  Theorem~\ref{Them-LRTOptimalityGeneral}, $ \Gamma_1^{v} = \Gamma_1^{opt}$ is the optimum decision region to minimize the probability of error. 
%The optimal threshold is defined as $\tau^v = \ln \frac{\pi_0}{\pi_1}$. 
Since $\Gamma_1^{\phi} \neq \Gamma_1^{v}$, there exists a set of observations $\mathcal{M} = (\Gamma_1^{\phi} \setminus \Gamma_1^{v}) \cup (\Gamma_1^{v} \setminus \Gamma_1^{\phi})$ that occur with nonzero probability (as bivariate Gaussian densities have support over $\mathbb{R}^2$) where the two decision regions $\Gamma_1^{\phi}$ and $\Gamma_1^{v}$ make different decisions. 
Due to this, the two tests will 
result in different probability of errors.\end{proof} 

Next, we present a theorem demonstrating that the Shapley value test is also different from the test using $v(1)$ when the observations follow what has been called a bivariate Laplacian distribution subject to a constant additive attack/anomaly at the first sensor. The two tests must also yield different error probabilities.

\begin{theorem} \label{Them-ShapleySuboptimalityBetaGaussian} 
Let the two sensor unattacked/nonanomalous sensor data $\mathbf{x} = [x_1, x_2]^T$ follow what has been called a {Bivariate Laplacian} pdf (a special case of elliptically
symmetric pdf)~\cite{6530654,gomez1998multivariate, johnson1987multivariate} under class $C_0$ 
given by
\begin{equation}
\begin{split}
\label{LaplaceG}
f(\mathbf{x} | C_0) = & {\frac{0.5 \, \Gamma(p/2)}{\pi^{p/2} \, \Gamma(p) \, 2^p} \frac{1}{\det(\boldsymbol{\Sigma})^{1/2}}} \\
& {\exp\left( -\frac{1}{2} \sqrt{\mathbf{x}^T \boldsymbol{\Sigma}^{-1} \mathbf{x}} \right)}
\end{split}
\end{equation}

where 
%$\boldsymbol{\Sigma}$ is defined as 
\begin{equation}
\boldsymbol{\Sigma} = \begin{pmatrix} 
\sigma_1^2 & \rho \sigma_1 \sigma_2 \\ 
\rho \sigma_1 \sigma_2 & \sigma_2^2 
\end{pmatrix}
\label{Sigmadef}
\end{equation}
with $\rho \in (-1, 1)$.

Assume a constant additive nonzero attack/anomaly  $A$ is added to $x_1$ (no attack/anomaly for $x_2$) under class $C_1$. Let the classifier to decide if there is an attack/anomaly at sensor 1 employ a log-likelihood ratio classifier where the test statistic 
\begin{equation}
\label{v1inT4} 
v(1) = \ln \left( \frac{\int_{-\infty}^{\infty} f(\mathbf{x} - \mathbf{a} | C_0) \, dx_2}{{\int_{-\infty}^{\infty} f(\mathbf{x} | C_0) \, dx_2}} \right)
\end{equation}
is compared to an optimum 
threshold $\tau^v$, where 
$\mathbf{a} = [A, 0]^T$. Assume $A,\rho \neq 0$. 
Then the Shapley function test is different from the test using $v(1)$ for deciding if the attack/anomaly occurs at sensor 1. 
The two tests also have different probability of errors. 
\end{theorem}

\begin{proof}
First, the employed test statistic in (\ref{v1inT4}) will only depend on $x_1$, since $x_2$ is integrated out.  Thus, as in Theorem~\ref{Them-ShapleySuboptimalityGaussian}, we 
must have the region $\Gamma_1^{v} = \{ [x_1, x_2]^T : v(1) > \tau^v \}$ which must have a boundary defined by $v(1) = \tau^v$. So the boundary is 
defined by the value of $x_1$ (similar to Theorem~\ref{Them-ShapleySuboptimalityGaussian}). 

 Note that since $x_2$ is not attacked/anomalous (similar to Theorem~\ref{Them-ShapleySuboptimalityGaussian})
\begin{equation}
\label{v2inT4} 
\begin{split}
v(2) &= \ln \left( \frac{{\int_{-\infty}^{\infty} f(\mathbf{x} | C_1) \, dx_2}}{{\int_{-\infty}^{\infty} f(\mathbf{x} | C_0) \, dx_2}} \right) = 0
\end{split}
\end{equation}For the optimum two dimensional classifier we find 
{\begin{equation}
\label{v12th3}
\begin{split}
v(1,2) &= \ln \frac{f(\mathbf{x} - \mathbf{a} | C_0)}{f(\mathbf{x} | C_0)} \\
&= \frac{1}{2} \left[ \sqrt{\mathbf{x}^T \boldsymbol{\Sigma}^{-1} \mathbf{x}} - \sqrt{(\mathbf{x}-\mathbf{a})^T \boldsymbol{\Sigma}^{-1} (\mathbf{x}-\mathbf{a})} \right]. 
\end{split}
\end{equation}}
Using (\ref{shap}) for the two dimensional case considered yields 
{\begin{eqnarray} 
\phi(x_1) &=& \frac{v(1) + v(1,2) - v(2)}{2} \nonumber \\ 
&=& \frac{1}{2} \ln \left( \frac{\int_{-\infty}^{\infty} f(\mathbf{x} - \mathbf{a} | C_0) \, dx_2}{\int_{-\infty}^{\infty} f(\mathbf{x} | C_0) \, dx_2} \right) \nonumber \\
&+& \frac{1}{4} \left[ \sqrt{\mathbf{x}^T \boldsymbol{\Sigma}^{-1} \mathbf{x}} - \sqrt{(\mathbf{x}-\mathbf{a})^T \boldsymbol{\Sigma}^{-1} (\mathbf{x}-\mathbf{a})} \right]
\end{eqnarray}
The first term in $\phi_1(\mathbf{x})$ is exactly one-half of the marginal log-likelihood ratio $v(1)$, which depends only on $x_1$ as $x_2$ is integrated out. Consequently, any dependence of $\phi_1(\mathbf{x})$ on $x_2$ must arise from the square-root difference term denoted as
\begin{equation}
\Delta(x_1, x_2) \triangleq \sqrt{\mathbf{x}^T \boldsymbol{\Sigma}^{-1} \mathbf{x}} - \sqrt{(\mathbf{x}-\mathbf{a})^T \boldsymbol{\Sigma}^{-1} (\mathbf{x}-\mathbf{a})}.
\label{delta}
\end{equation}

To show that $\phi_1(\mathbf{x})$ depends on $x_2$, we provide a proof by contradiction. Assume that 
%for some fixed $x_1$, 
$\Delta(x_1, x_2)$ is not a function of $x_2$. If this were true, it must also hold for the specific case $x_1 = 0$. Substituting $x_1 = 0$ and $\mathbf{a} = [A, 0]^T$, and noting that 
$\boldsymbol{\Sigma}^{-1} = \frac{1}{(1-\rho^2)} \begin{pmatrix} \sigma_1^{-2} & -\rho(\sigma_1\sigma_2)^{-1} \\ -\rho(\sigma_1\sigma_2)^{-1} & \sigma_2^{-2} \end{pmatrix}$, we have 
\begin{equation}
\Delta(0, x_2) = \frac{1}{\sqrt{1-\rho^2}} \left[ \frac{|x_2|}{\sigma_2} - \sqrt{\frac{A^2}{\sigma_1^2} + \frac{2\rho A}{\sigma_1 \sigma_2}x_2 + \frac{x_2^2}{\sigma_2^2}} \right].
\end{equation}
If $\Delta(0, x_2)$ is not a function of $x_2$, there must exist a constant $C$ such that for all $x_2 \ge 0$
\begin{equation}
\frac{x_2}{\sigma_2} - \sqrt{\frac{A^2}{\sigma_1^2} + \frac{2\rho A}{\sigma_1 \sigma_2}x_2 + \frac{x_2^2}{\sigma_2^2}} = C.
\end{equation}
Rearranging and squaring both sides yields
\begin{equation}
\frac{A^2}{\sigma_1^2} + \frac{2\rho A}{\sigma_1 \sigma_2}x_2 + \frac{x_2^2}{\sigma_2^2} = \left( \frac{x_2}{\sigma_2} - C \right)^2 = \frac{x_2^2}{\sigma_2^2} - \frac{2C}{\sigma_2}x_2 + C^2.
\end{equation}
Matching the coefficients of $x_2$ on both sides requires $C = -\frac{\rho A}{\sigma_1}$. Substituting this into the constant term gives
\begin{equation}
\frac{A^2}{\sigma_1^2} = C^2 = \frac{\rho^2 A^2}{\sigma_1^2}.
\end{equation}
For $A \neq 0$, this identity implies $\rho^2 = 1$, which contradicts the assumption that $\rho \in (-1, 1)$. Thus, the assumption that $\Delta(x_1, x_2)$ is not a function of $x_2$ is false.

It follows that $\phi_1(\mathbf{x})$ depends on both $x_1$ and $x_2$ so that the boundary of the decision region $\Gamma_1^{\phi} = \{ [x_1, x_2]^T : \phi_1(\mathbf{x}) > \tau^s \}$ must depend on both $x_1$ and $x_2$ and so this region must be different from the region $\Gamma_1^{v}$  on sets which occur with nonzero probability, since our pdfs have support over the whole two dimensional space.  Thus the decisions made by the Shapley function test are different from those made by the $v(1)$ test and the two tests will have different error %\textcolor{red} 
{probabilities}.

}

\end{proof}
}

Finally, we establish a theorem demonstrating that the Shapley value test must be different from the test using $v(1)$ for bivariate Gaussian observations subject to a random Gaussian attack/anomaly at sensor 1. 
{\begin{theorem}\label{Them-ShapleySuboptimalityGaussianRandom}Consider the case where the unattacked/nonanomalous  (class $C_0$) sensor data $\mathbf{x} = [x_1, x_2]^T$ follows a bivariate Gaussian distribution with zero mean vector and covariance matrix 
$\boldsymbol{\Sigma}$
defined in (\ref{Sigmadef}) . 
{Under $C_1$, assume a two dimensional random additive attack/anomaly vector $\mathbf{a}=[A,0]^T$ is added to the unattacked/nonanomalous vector, where $A \sim \mathcal{N}(\mu_A, \sigma_A^2)$. Thus the random Gaussian attack/anomaly is applied only to $x_1$.}  
Then the Shapley function test using $\phi(x_1)$ must be different from the test using $v(1)$ if $\rho \neq 0$, regardless of the value of $\mu_A$. The two tests will achieve different error {probabilities}.  
\end{theorem} }
\begin{proof} {  Under $C_1$, 
$\mathbf{x}$ has mean vector $ [\mu_A, 0]^T$ and covariance matrix $\boldsymbol{\Sigma}'$, where 
\begin{equation} \label{sigprime} \begin{aligned}
\boldsymbol{\Sigma}' &= \boldsymbol{\Sigma} + \begin{pmatrix} \sigma_A^2 & 0 \\ 0 & 0 \end{pmatrix} \\
&= \begin{pmatrix} \sigma_1^2 + \sigma_A^2 & \rho \sigma_1 \sigma_2 \\ \rho \sigma_1 \sigma_2 & \sigma_2^2 \end{pmatrix}.
\end{aligned}
\end{equation}
Let $\Gamma_1^{\phi} = \{ [x_1, x_2]^T : \phi(x_1) > \tau^s \}$ and $\Gamma_1^{v} = \{ [x_1, x_2]^T : v(1) > \tau^v \}$, where $\tau^s$ and $\tau^v$ are the optimum thresholds (minimum error probability) for the two tests. 
For the two tests to be equivalent, then $\Gamma_1^{\phi}$ must be equivalent to $\Gamma_1^{v}$, , implying $\phi(x_1)$ is not a function of $x_2$. We will show this cannot be true. 
Thus, we again use a proof by contradiction. 

Using (\ref{shap}) for the two-dimensional case yields 

% \begin{equation}\phi(x_1) = (v(1) + v(1,2) - v(2))/2.\label{80}\end{equation}
{
\begin{equation}
\phi(x_1 ) = \frac{v(1 ) + v(1,2 ) - v(2 )}{2}.
\label{80}
\end{equation}
}

Since $x_2$ is not attacked/anomalous, {$v(2) = 0$} (see (\ref{v2inT4})), and {$v(1)$} only depends on $x_1$ since {\begin{equation}
v(1) = \ln \frac{f(x_1 | C_1)}{f(x_1 | C_0)}.
\end{equation}} Thus, any dependency of the Shapley function on $x_2$ can only arise from {$v(1,2 )$} given (\ref{80}). Now, the joint log-likelihood ratio, $v(1,2)$, is defined as
$ v(1,2) = $
{\begin{equation}
%v(1,2) = 
\ln \left( \frac{\frac{1}{2\pi |\boldsymbol{\Sigma}'|^{1/2}} \exp\left[ -\frac{1}{2} (\mathbf{x}-[\mu_A, 0]^T )^T {\boldsymbol{\Sigma}'}^{-1} (\mathbf{x}-[\mu_A, 0]^T ) \right]}{\frac{1}{2\pi |\boldsymbol{\Sigma}|^{1/2}} \exp\left[ -\frac{1}{2} \mathbf{x}^T \boldsymbol{\Sigma}^{-1} \mathbf{x} \right]} \right),
\end{equation}
}
which is equivalent to $ v(1,2) = $ 
{
\begin{equation}
%v(1,2 | \mathbf{a}) = \frac{1}{2} 
\ln \frac{|\boldsymbol{\Sigma}|}{|\boldsymbol{\Sigma}'|} + \frac{1}{2} \left[ \mathbf{x}^T \boldsymbol{\Sigma}^{-1} \mathbf{x} - (\mathbf{x}-[\mu_A, 0]^T)^T {\boldsymbol{\Sigma}'}^{-1} (\mathbf{x}- [\mu_A, 0]^T ) \right].\label{vec} \end{equation}}
%{\color{green}The fix you have added to the overleaf (above and below) is not correct. Major issue: a is not equal to $[\mu_a,0]$ as said after (34). $\mu_a$ is the expected value a. Also you did not introduce the $v(\cdot |a)$ notation. Further you need to prove v(1,2) is not a function of $x_2$, not $v(1,2|a)$. }
Expanding (\ref{vec}) yields 
% obtain the closed-form expression for 
{\begin{equation}
\begin{aligned}
v(1,2) &= \frac{1}{2} \ln \frac{|\boldsymbol{\Sigma}|}{|\boldsymbol{\Sigma}'|} \\
&\quad + \frac{1}{2} ((\boldsymbol{\Sigma}^{-1})_{11} - ({\boldsymbol{\Sigma}'}^{-1})_{11})x_1^2 \\
&\quad + \frac{1}{2} ((\boldsymbol{\Sigma}^{-1})_{22} - ({\boldsymbol{\Sigma}'}^{-1})_{22})x_2^2 \\
&\quad + ((\boldsymbol{\Sigma}^{-1})_{12} - ({\boldsymbol{\Sigma}'}^{-1})_{12})x_1 x_2 \\
&\quad + \mu_A ({\boldsymbol{\Sigma}'}^{-1})_{11}x_1 \\
& + \mu_A ({\boldsymbol{\Sigma}'}^{-1})_{12}x_2 - \frac{1}{2} \mu_A^2 ({\boldsymbol{\Sigma}'}^{-1})_{11},
\end{aligned}
\end{equation}}where $(\boldsymbol{\Sigma}^{-1})_{22} = \frac{\sigma_1^2}{|\boldsymbol{\Sigma}|} = \frac{1}{\sigma_2^2(1-\rho^2)}$ and $({\boldsymbol{\Sigma}'}^{-1})_{22} = \frac{\sigma_1^2 + \sigma_A^2}{|\boldsymbol{\Sigma}'|}$ from 
(\ref{sigprime}). 
In order for {$\phi(x_1 )$} to not be a function of $x_2$, we 
need all the coefficients multiplying $x_2^2$, $x_2$ and $x_1 x_2$ to be zero.  This requires ($x_2^2$ case) $(\boldsymbol{\Sigma}^{-1})_{22} - ({\boldsymbol{\Sigma}'}^{-1})_{22} = 0$, which is equivalent to
\begin{equation}\frac{1}{\sigma_2^2(1-\rho^2)} = \frac{\sigma_1^2 + \sigma_A^2}{|\boldsymbol{\Sigma}'|}.\end{equation} Substituting $|\boldsymbol{\Sigma}'| = \sigma_2^2(\sigma_1^2 + \sigma_A^2 - \rho^2 \sigma_1^2)$, the condition becomes
\begin{equation}\frac{1}{\sigma_2^2 (1 - \rho^2)} = \frac{\sigma_1^2 + \sigma_A^2}{\sigma_2^2 (\sigma_1^2 + \sigma_A^2 - \rho^2 \sigma_1^2)}.\end{equation}

Cross-multiplying yields $\sigma_1^2 + \sigma_A^2 - \rho^2 \sigma_1^2 = \sigma_1^2 + \sigma_A^2 - \rho^2 \sigma_1^2 - \rho^2 \sigma_A^2$, which simplifies to $0 = \rho^2 \sigma_A^2$. Under the assumed conditions ($\sigma_A^2 > 0$ is always true) $\rho \neq 0$, this is a contradiction. 
%Similarly, the coefficient linear term $\mu_A ({\boldsymbol{\Sigma}'}^{-1})_{12}x_2$ also persists unless $\rho=0$. 
Thus, {$\phi(x_1)$} depends on $x_2$, making its decision boundary deviate from the  boundary $v(1)=\tau^v$, proving the two tests are different. These differences occur on sets which occur with nonzero probability since our pdfs have support over the whole two dimensional space. This implies the error probabilities of the two tests will be different.} \end{proof}

We note that while we have focused on attacks/anomalies only at sensor 1, similar results can be obtained for attacks/anomalies only at sensor 2 by symmetry.

%{\color{red} add new theorems showing Shapley better or inferior in certain cases in next subsection 2).}

\subsubsection{Comparing Shapley and $v(1)$ Tests}
{
\begin{theorem} \label{Them-ShapleyBetterOrInferior} 
Let the two sensor unattacked/nonanomalous sensor data follow a zero mean ($\mu_1=\mu_2=0$ in (\ref{BG})), common standard deviation 
%\footnote{Such restrictions can be removed, but will result in a more complicated theorem.} 
($\sigma_1=\sigma_2=\sigma$ in (\ref{BG})) bivariate Gaussian distribution under class $C_0$.  
Starting with data following $C_0$, a constant additive nonzero attack/anomaly $A$  is added to the first sensor observation to arrive at the data following $C_1$.  Assume 
%$A > 0$,
$|\rho| $ is close to but slightly less than $1 $ and let $C_0$ and $C_1$ be equally likely so $\pi_0=\pi_1=1/2$. Let $P_{e,\phi}$ and $P_{e,v}$ denote the minimum probability of error for the Shapley function test and the test using $v(1)$, respectively. 
\begin{enumerate}
    \item If $\rho $ is close to but slightly less than $ 1 $, the Shapley function test is strictly superior to the $v(1)$ test, such that $P_{e,\phi} < P_{e,v}$.
    %for a sufficiently large attack magnitude $|A|$.
    \item If $\rho $ is close to but slightly greater  than $ -1 $, the Shapley function test is  strictly inferior to the $v(1)$ test, such that $P_{e,\phi} > P_{e,v}$.
    % for a sufficiently large attack magnitude $|A|$.
\end{enumerate}
\end{theorem}

\begin{proof}
{First assume $A>0$. 
Let $\Phi(x)$ denote the standard Gaussian (zero mean and unit variance) cumulative distribution function (cdf).  First consider the test using $v(1)$. 
From (\ref{T3.3v1}), 
comparing $v(1)$ to a threshold for Gaussian $x_1$ is equivalent to comparing $x_1$ to a threshold. 
Under $C_1$,  $x_1 = A + n_1 $ where $n_1$ denotes a zero mean and  standard deviation $\sigma$ Gaussian random variable. Under $C_0$, $x_1 = n_1 $.  
Since the two Gaussian pdfs of $x_1$ under $C_0$ and $C_1$, both with the same standard deviation, are centered at $0$ and $A$ under the two equally likely classes, then the optimum threshold is midway in between at $A/2$. 
This follows from the proof of Theorem~\ref{Them-LRTOptimalityGeneral} 
where it is shown that for equal priors, $\pi_0=\pi_1$, the optimum test decides for the class of the larger of the two conditional pdfs 
$ f({\bf x}|C_1), f({\bf x}|C_0)$, thus 
the threshold $x_1$ is compared to should be midway in between the peaks of the two pdfs, see Fig.~\ref{fig:error_analysis}.

\begin{figure}[H]    
\centering    
\includegraphics[width=0.5\textwidth]{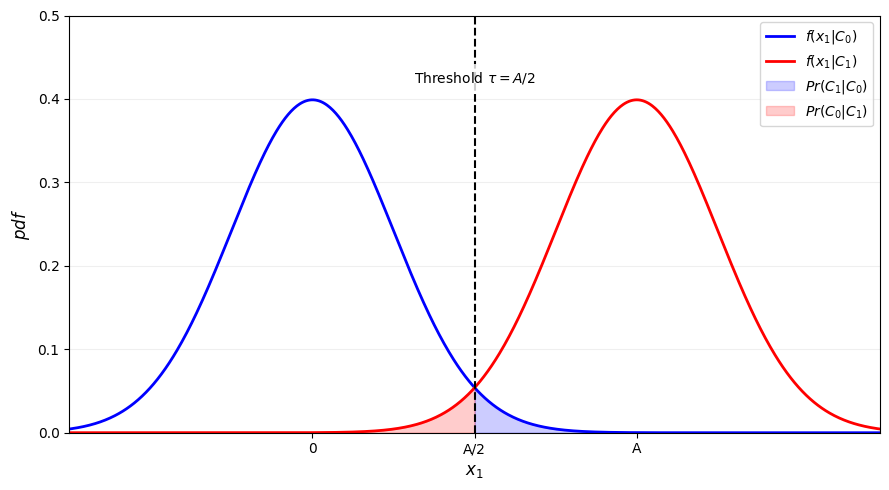} 
\caption{Statistical error probabilities for the binary hypothesis test discussed in the beginning of Theorem~\ref{Them-ShapleyBetterOrInferior}. 
% {\color{red} FIX THIS IN FIG NO PFA OR PMISSDET, ALSO THE BLUE AND RED CURVES ARE $f(x_1|C_i), i=0,1$: 
% The shaded areas represent $Pr(C_1|C_0)$ 
% and $Pr(C_0|C_1)$ defined in 
% Theorem~\ref{Them-LRTOptimalityGeneral}}
}.
\label{fig:error_analysis}\end{figure}

As per Fig.~\ref{fig:error_analysis} 
%Considering $A>0$ gives 
\begin{eqnarray} 
P_{e,v} &=& \frac{1}{2}  Pr( A + n_1 < \frac{A}{2} | C_1) 
+
 \frac{1}{2}  Pr( n_1 >  \frac{A}{2} | C_0) \nonumber \\
&=& \frac{1}{2} Pr( n_1 < -\frac{A}{2} | C_1) + 
\frac{1}{2} Pr( n_1 >   \frac{A}{2 \sigma} | C_0) \nonumber \\ 
&=& \frac{1}{2} \Phi\left( - \frac{A}{2 \sigma} \right) + 
\frac{1}{2} \left( 1- \Phi(\frac{A}{2 \sigma})\right) 
 \label{proof:pev1}
\end{eqnarray}
and $P_{e,v}$ is a decreasing function of $ \frac{A}{2 \sigma} $ since a cdf is always an increasing function when the random variable involved has support over the real line. 
%{\ Picture above will also explain this.}

For the Shapley function test (still assuming $A>0$), we rewrite the Shapley function from 
(\ref{eq18}) ($\sigma_1=\sigma_2=\sigma, \mu_1=\mu_2=0$) in a form that isolates the difference between the two sensor observations to obtain 
\begin{equation}
\phi(x_1) = a x_1 + b (x_1 - x_2) - \text{Constant},
\label{eq:phi_split_form}
\end{equation}
where \begin{equation}
    a = \frac{A(2-\rho^2 - \rho)}{2(1-\rho^2)\sigma^2}
    \label{a}
\end{equation}
and 
\begin{align}
b = \frac{A\rho}{2(1-\rho^2)\sigma^2}.
\label{b}
\end{align}
Note that $ a,b > 0 $ under the given conditions.
Since the constant term in (\ref{eq:phi_split_form}) can be incorporated into the threshold $\tau^s$ in the test and this threshold will be optimized, the constant term 
in (\ref{eq:phi_split_form}) 
is not important. 
If we denote the overall optimized threshold, which combines $\tau^s$ and the constant in (\ref{eq:phi_split_form}) as $\tau'$, then the test can be expressed as 
\begin{equation}
  \phi(x_1) =  a x_1 + b (x_1 - x_2) \mathop{\gtrless}_{H_0}^{H_1} \tau'
\label{phix1andb}\end{equation}

First consider the case where $\rho$ is close to $1$. From (\ref{phix1andb}) and under $C_1$, $\phi(x_1) \approx  a(A+n_1) + bA = (a+b)A + a n_1$  since $n_1 \approx n_2$.  Under $C_0$, $\phi(x_1) \approx a n_1$ since $n_1 \approx n_2$. 
Since the two Gaussian pdfs, both with the same standard deviation, are centered at $0$ and $(a+b)A $ under the two equally likely classes, then (similar to the $v(1)$ case) the optimum threshold is midway in between at $(a+b)A /2$. 
Thus 
\begin{eqnarray} 
P_{e,\phi} &=& \frac{1}{2}  Pr( (a+b)A + a n_1 < \frac{(a+b)A }{2} | C_1) \nonumber \\
&+& 
 \frac{1}{2}  Pr( a n_1 >  \frac{(a+b)A}{2} | C_0) \nonumber \\
&=& \frac{1}{2}  Pr( (a n_1 < - \frac{(a+b)A }{2} | C_1) \nonumber \\
&+& 
 \frac{1}{2}  Pr( a n_1 >  \frac{(a+b)A}{2} | C_0) \nonumber \\
 &=& \frac{1}{2}  \Phi\left(  - \frac{(a+b)A }{2a \sigma} \right) \nonumber \\
&+& 
 \frac{1}{2}  \left( 1 - \Phi\left(\frac{(a+b)A}{2 a \sigma}\right) \right) \nonumber \\ 
 \label{proof:pephi1}
\end{eqnarray}
which (similar to the $v(1)$ case) is a decreasing function of $\frac{(a+b)A}{2 a \sigma}$. Comparing (\ref{proof:pev1}) and (\ref{proof:pephi1}), the Shapley test is better (based on being a decreasing function of this quantity) if 
$\frac{(a+b)A}{2 a \sigma} > \frac{A}{2 \sigma}$ which is always true since 
$\frac{(a+b)}{a}>1$ when $a,b > 0 $. 

If $\rho$ is close to but slightly greater than $-1$, then under $C_1$, $\phi(x_1) \approx a(A+n_1) + b(A +2 n_1)  = (a+b)A + (a+2b)n_1$ since $ n_2 \approx -n_1 $.   Under $C_0$, $\phi(x_1) = a n_1  + b(2 n_1) $ since $ n_2 \approx -n_1 $. This time the two pdfs are centered at $0$ and $(a+b)A$ so the optimum threshold is midway in between. Thus 
\begin{eqnarray} 
P_{e,\phi} &=& \frac{1}{2}  Pr( (a+b)A + (a + 2b) n_1 < \frac{(a+b)A }{2} | C_1) \nonumber \\
&+& 
 \frac{1}{2}  Pr(  (a + 2b)n_1 >  \frac{(a+b)A}{2 } | C_0) \nonumber \\
&=& \frac{1}{2}  Pr( ( n_1 < - \frac{(a+b)A }{2(a+ 2b)} | C_1) \nonumber \\
&+& 
 \frac{1}{2}  Pr( n_1 >  \frac{(a+b)A}{2(a + 2b)} | C_0) \nonumber \\
 &=& \frac{1}{2}  \Phi\left( - \frac{(a+b)A }{2 (a + 2b)  \sigma} \right) \nonumber \\
&+& 
 \frac{1}{2}  \left( 1 - \Phi\left(\frac{(a+b)A}{2 (a + 2b) \sigma}\right) \right) \nonumber \\ 
 \label{proof:pephi2}
\end{eqnarray}
Comparing (\ref{proof:pev1}) and (\ref{proof:pephi2}), the Shapley test is inferior if 
$\frac{(a+b)A}{2 (a + 2b) \sigma} < \frac{A}{2 \sigma}$ which is always true since 
$\frac{(a+b)}{(a + 2b)}<1$ when $a,b > 0 $.

A similar proof can be given if $A<0$, which we skip for brevity.

}
\end{proof}

To compute $a$ and $b$ 
from (\ref{a}) and (\ref{b}), we need the value of $\rho$. Thus assume $\rho$ is known (or can be learned as it is in Shapley).  Then, for high correlation magnitude ($|\rho|$ near $1$), one can 
employ an approach that switches between the two approaches (Shapley and  $v(1)$) 
considered in Theorem~\ref{Them-ShapleyBetterOrInferior} based on the sign of $\rho$ and this approach will be strictly better than either approach in these cases. 
%If the sign of $\rho$ is unknown, one might try to estimate it. 
}

\section{Numerical Results}
%  \textbf{\textcolor{red}{
%  % Parameters needed for the plots caption, 
%  also give discussion for the plots,This should come before the discussion of Table 2 since the
% related theorems are first. .Add the case of rho= ±.9.
% }}

\subsubsection{Statistically Independent Observations}

{To further demonstrate the results} in Theorem III.1, we numerically compare the probability of error $P_e$ of a test that compares $\phi(x_i)$ to {the optimal threshold} to that for a test that compares $v(i)$ to {the optimal threshold}. The two tests each make a decision on if the attack/anomaly includes the $i$-th sensor. The better test will have a smaller $P_e$, {indicating which statistic} is better for localizing the attack/anomaly. %\textcolor{green}{Since both $v(x_i)$ and $\phi(x_i)$ are derived from the Log-Likelihood Ratio (LLR), the zero threshold minimizes the total $P_e$ under equal prior probabilities. }
In our numerical results, we use a Monte Carlo simulation {with $M=10^{7}$ runs} to approximate the probability of error. Let the symbols $P_{e,\phi}$ and $P_{e,v}$ denote the {minimum achievable} probability of error for the test using $\phi(x_i)$ and {the test using} $v(i)$, respectively.

In {Table~\ref{tab:two_sensor_results_updated}}, we present results obtained from running a Monte Carlo simulation with {$M= 10^{7}$} runs. {We also analytically calculated the numerical results and compared them to the counterpart obtained from Monte Carlo Simulations. We found  they were very close in all cases, which helps justify the accuracy of these numerical results.}{ We generate the non-anomalous data sample $(x_1, x_2)$ as having 
statistically independent components with each component being Gaussian and  having zero mean and standard deviation $\sigma$ as given in 
{Table~\ref{tab:two_sensor_results_updated}}. 
We let half of the $M$ Monte Carlo runs have sensor {1} be anomalous while the other half do not. 
Sensor 2 is always not anomalous. The minimum error probabilities are found from a grid search over a fine grid.}

{In {Table~\ref{tab:two_sensor_results_updated}}, we consider three distinct types of anomaly/attack  models to evaluate the localization performance of $v(i)$ and $\phi(x_i)$. For \textbf{Type A}, a constant value called the anomaly/attack magnitude, denoted by %\textcolor{red}
{$A$}, is added to the non-anomalous/unattacked  observation at sensor 1. For \textbf{Type B}, a Gaussian random variable is added to the non-anomalous/unattacked  observation. The added  %\textcolor{red}
{Gaussian} variable has a mean of %\textcolor{red}
{$A$} and a standard deviation $\sigma_a$. For \textbf{Type C}, a uniform random variable is added to the non-anomalous/unattacked  observation. The added uniform random variable is defined by the sum of a constant %\textcolor{red}
{$A$} and 
a random variable 
uniformly distributed over the range $ [-UM/2,UM/2]$. 

}

The results in Table~\ref{tab:two_sensor_results_updated} show that in all cases considered, we find $P_{e,\phi} = P_{e,v}$ which agrees with the major conclusions of 
Theorem III.1, which says the two tests must be the same and achieve the same probability of error.  The numerical results for $P_{e,\phi} $ and $ P_{e,v}$, closely matched numerical calculations, but they also exhibit reasonable trends, for example they decrease with an increase in %\textcolor{red} 
{$A$}.

\subsubsection{Statistically Dependent Observations}
In Theorems III.3 through III.5, we show the two compared tests, the one using Shapley ($\phi(x_1)$) and the one using $v(1)$, are different, by essentially showing the decision regions are different for sets of observations that occur with nonzero probability. To illustrate these ideas with a figure, we focus on the proof of Theorem~\ref{Them-ShapleySuboptimalityBetaGaussian} where we essentially show that $\phi(x_1)$ must always depend on $x_2$, while $v(1)$ must always not depend on $x_2$.  This is why the decision regions, which are obtained by comparing these functions to a constant threshold, must be different. It is clear that $v(1)$ does not depend on $x_2$ since $x_2$ is integrated out in (22). While (25) shows that $\phi(x_1)$ can depend on $x_2$ through the quantity in (\ref{delta}), 
it is not obvious it must depend on  $x_2$ so the proof shows this is the case. To provide  intuition, we plot the quantity $\Delta(x_1,x_2)$ in (25) versus $x_2$ for the special case of {
$A=1.0$, $\sigma_1=1.0$, $\sigma_2=1.2$, $\rho = 0.5$ }and several values of $x_1$ (see the Fig.) in Figure~\ref{fig:delta_slice}. Based on Figure~\ref{fig:delta_slice}, 
$\Delta(x_1,x_2)$ varies with $x_2$ for all the cases shown, 
which agrees with the ideas stated in the proof of 
Theorem~\ref{Them-ShapleySuboptimalityBetaGaussian}.  

\begin{table}[htbp]
\centering
\caption{Minimum $P_e$ for $v(1)$ and $\phi(x_1)$. See text for details. }
\label{tab:two_sensor_results_updated}
\begin{small}
\setlength{\tabcolsep}{3.2pt} % 略微压缩以适配 8 行 Type B 带来的纵向视觉比例
\begin{tabular}{@{}ccccccc@{}}
\toprule
$\sigma$ & \textbf{Type} & $\sigma_a$ & \textbf{A} & \textbf{UM} & $P_{e,v}$ & $P_{e,\phi}$ \\ \midrule

% Type A: 4 Cases Total
1.0 & A & na & 1.0 & na & $3.08955 \times 10^{-1}$ & $3.08955 \times 10^{-1}$ \\
2.0 & A & na & 1.0 & na & $4.01197 \times 10^{-1}$ & $4.01197 \times 10^{-1}$ \\
1.0 & A & na & 5.0 & na & $6.24600 \times 10^{-3}$ & $6.24600 \times 10^{-3}$ \\
2.0 & A & na & 5.0 & na & $1.05552 \times 10^{-1}$ & $1.05552 \times 10^{-1}$ \\ \bottomrule

% 已有 Case: AM=1.0
1.0 & B & 0.1 & 1.0 & na & $3.08939 \times 10^{-1}$ & $3.08939 \times 10^{-1}$ \\
2.0 & B & 0.1 & 1.0 & na & $4.01180 \times 10^{-1}$ & $4.01180 \times 10^{-1}$ \\
1.0 & B & 1.0 & 1.0 & na & $3.27327 \times 10^{-1}$ & $3.27327 \times 10^{-1}$ \\ 
2.0 & B & 1.0 & 1.0 & na & $4.04178 \times 10^{-1}$ & $4.04178 \times 10^{-1}$ \\ \midrule
% 新增 Case: AM=5.0
1.0 & B & 0.1 & 5.0 & na & $6.30500 \times 10^{-3}$ & $6.30500 \times 10^{-3}$ \\ 
2.0 & B & 0.1 & 5.0 & na & $1.05779 \times 10^{-1}$ & $1.05779 \times 10^{-1}$ \\ 
1.0 & B & 1.0 & 5.0 & na & $1.87930 \times 10^{-2}$ & $1.87930 \times 10^{-2}$ \\ 
2.0 & B & 1.0 & 5.0 & na & $1.18829 \times 10^{-1}$ & $1.18829 \times 10^{-1}$ \\ \bottomrule

% Type C: 4 Cases Total
1.0 & C & na & 1.0 & 0.1 & $3.08990 \times 10^{-1}$ & $3.08990 \times 10^{-1}$ \\
2.0 & C & na & 1.0 & 0.1 & $4.01662 \times 10^{-1}$ & $4.01662 \times 10^{-1}$ \\ 
1.0 & C & na & 5.0 & 0.1 & $6.28700 \times 10^{-3}$ & $6.28700 \times 10^{-3}$ \\ 
2.0 & C & na & 5.0 & 0.1 & $1.05842 \times 10^{-1}$ & $1.05842 \times 10^{-1}$ \\ \bottomrule
\end{tabular}
\end{small}
\end{table}

\begin{table}[htbp]
\centering
\small
\setlength{\tabcolsep}{1pt} 
\caption{Comparison of $P_e$ for $v_1$ and $\phi(x_1)$. Attack Type A: Constant bias $A$. Attack Type B: Gaussian attack/anomaly $a \sim N(\mu_A, \sigma_A^2)$. $\sigma$ is noise standard deviation.}
\label{tab:dependent_results}
\renewcommand{\arraystretch}{1.3}

\begin{tabular*}{\columnwidth}{|@{\extracolsep{\fill}}l|l|l|c|c|@{}}
\hline
\textbf{Dist.} & \textbf{Att.} & \textbf{Parameters} & $P_{e,v}$ & $P_{e,\phi}$ \\ \hline
\multirow{24}{*}{\shortstack[l]{\textit{Scen. 1:}\\ Biv. Gau.}} 
 & A & $\rho=0.8, \sigma=2, A=1.5$  & 0.3537 & 0.1745 \\
 & A & $\rho=0.8, \sigma=2, A=3.0$  & 0.2265 & 0.0305 \\
 & A & $\rho=0.8, \sigma=2, A=5.0$  & 0.1055 & 0.0009 \\ \cline{2-5}
 & A & $\rho=0.9, \sigma=2, A=1.5$  & 0.3537 & 0.0756 \\ % 新增 0.9
 & A & $\rho=0.9, \sigma=2, A=3.0$  & 0.2265 & 0.0020 \\
 & A & $\rho=0.9, \sigma=2, A=5.0$  & 0.1055 & 0.0000 \\ \cline{2-5}
 & A & $\rho=0.8, \sigma=2, A=-1.5$ & 0.3539 & 0.1746 \\
 & A & $\rho=0.8, \sigma=2, A=-3.0$ & 0.2266 & 0.0305 \\
 & A & $\rho=0.8, \sigma=2, A=-5.0$ & 0.1056 & 0.0009 \\ \cline{2-5}
 & A & $\rho=0.9, \sigma=2, A=-1.5$ & 0.3539 & 0.0757 \\ % 新增 0.9
 & A & $\rho=0.9, \sigma=2, A=-3.0$ & 0.2266 & 0.0020 \\
 & A & $\rho=0.9, \sigma=2, A=-5.0$ & 0.1056 & 0.0000 \\ \cline{2-5}
 & A & $\rho=-0.8, \sigma=2, A=1.5$ & 0.3537 & 0.4040 \\
 & A & $\rho=-0.8, \sigma=2, A=3.0$ & 0.2265 & 0.3136 \\
 & A & $\rho=-0.8, \sigma=2, A=5.0$ & 0.1055 & 0.2092 \\ \cline{2-5}
 & A & $\rho=-0.9, \sigma=2, A=1.5$ & 0.3537 & 0.4211 \\ % 新增 -0.9
 & A & $\rho=-0.9, \sigma=2, A=3.0$ & 0.2265 & 0.3452 \\
 & A & $\rho=-0.9, \sigma=2, A=5.0$ & 0.1055 & 0.2534 \\ \cline{2-5}
 & A & $\rho=-0.8, \sigma=2, A=-1.5$ & 0.3539 & 0.4042 \\
 & A & $\rho=-0.8, \sigma=2, A=-3.0$ & 0.2266 & 0.3138 \\
 & A & $\rho=-0.8, \sigma=2, A=-5.0$ & 0.1056 & 0.2092 \\ \cline{2-5}
 & A & $\rho=-0.9, \sigma=2, A=-1.5$ & 0.3539 & 0.4213 \\ % 新增 -0.9
 & A & $\rho=-0.9, \sigma=2, A=-3.0$ & 0.2266 & 0.3454 \\
 & A & $\rho=-0.9, \sigma=2, A=-5.0$ & 0.1056 & 0.2536 \\ \hline

\multirow{12}{*}{\shortstack[l]{\textit{Scen. 2:}\\ Biv. Gau.}} 
 & B & $\rho=0.8, \sigma=1, \sigma_A=1, A=1.5$ & 0.2597 & 0.1207 \\
 & B & $\rho=0.8, \sigma=1, \sigma_A=1, A=3.0$ & 0.1049 & 0.0100 \\
 & B & $\rho=0.8, \sigma=1, \sigma_A=1, A=5.0$ & 0.0189 & 0.0001 \\ \cline{2-5}
 & B & $\rho=0.9, \sigma=1, \sigma_A=1, A=1.5$ & 0.2597 & 0.0914 \\ % 新增 0.9
 & B & $\rho=0.9, \sigma=1, \sigma_A=1, A=3.0$ & 0.1049 & 0.0039 \\
 & B & $\rho=0.9, \sigma=1, \sigma_A=1, A=5.0$ & 0.0189 & 0.0000 \\ \cline{2-5}
 & B & $\rho=-0.8, \sigma=1, \sigma_A=1, A=1.5$ & 0.2597 & 0.3683 \\
 & B & $\rho=-0.8, \sigma=1, \sigma_A=1, A=3.0$ & 0.1049 & 0.2670 \\
 & B & $\rho=-0.8, \sigma=1, \sigma_A=1, A=5.0$ & 0.0189 & 0.1549 \\ \cline{2-5}
 & B & $\rho=-0.9, \sigma=1, \sigma_A=1, A=1.5$ & 0.2597 & 0.4163 \\ % 新增 -0.9
 & B & $\rho=-0.9, \sigma=1, \sigma_A=1, A=3.0$ & 0.1049 & 0.3503 \\
 & B & $\rho=-0.9, \sigma=1, \sigma_A=1, A=5.0$ & 0.0189 & 0.2662 \\ \hline

\end{tabular*}
\end{table}

{
\color{red}

% {\color{red} What does the next part mean?  In Table 1 part I described nonanomalous first and then anomalous?  It is very unclear.  What is adversarial conditions? Use the same type of language as Table 1 discussion.  You are not evaluating the impact of correlation, don't say that.  You are considering cases with large $|\rho|$ and in particualar $.8,.9$.  Doesn't the Table give parameters.  Whay are you saying this.  Writing is long and boring and not even complete. Need to say nonanomalous has zero mean.  You repeat.  You describe the results of comparing $P_e$s at the beginning and the end?  
% Why do you say type B is reverse?  Seems the same to me? When $\rho=.8$, Shapley better. 
% When $\rho=-.8$, Shapley inferior. 
% Are we sure all these results are correct? Are the calculated?  MC? MC runs not stated?  Are all parameters stated? 
% You don't vary constant attack mag from $-5$ to $5$, you try some specific values given in the Table 2. Can you write this is a way that fits?  Is coherent, clear, and as short as possible, but where every needed detail is given.

% It seems you did not write anything to describe the figures?  This should come before the discussion of Table 2 since the related theorems are first.  It should be obvious from $\phi$ that $\rho$ will impact it (see last theorem).  It is also clear it will not impact $v(1)$.  It does not even involve both variables.  Think about what you are writing.} 
}
{Finally we present numerical results which compare  $P_{e,\phi}$ and $P_{e,v}$ for cases with statistically dependent observations. 
In Table~\ref{tab:dependent_results}, we present results obtained from running a Monte Carlo simulation with $M = 10^7$ runs. We generate the non-anomalous/unattacked  sensor observations $(x_1, x_2)$ following a Bivariate Gaussian pdf in (\ref{BG}) with $ \mu_1=\mu_2=0, \sigma_1=\sigma_2= \sigma$ for the values of $\sigma$ and $\rho$ given in Table~\ref{tab:dependent_results}. We let half of the M Monte Carlo runs have sensor 1 be anomalous/attacked, while the other half do not. Sensor 2 is always not anomalous/attacked. The minimum error probabilities $P_{e,\phi}$ and $P_{e,v}$  are found
from a grid search over a fine grid.

We consider 2 types of anomaly models to evaluate the localization performance of $v(i)$ and $\phi(x_i)$. For \textbf{Type A}, a constant value called the anomaly/attack magnitude, denoted by {$A$}, is added to the non-anomalous/unattacked observation at sensor 1. For \textbf{Type B}, a Gaussian random variable is added to the non-anomalous/unattacked  observation. The added Gaussian variable has a mean of {$A$} and a standard deviation $\sigma_a$. In Table~\ref{tab:dependent_results}, we present results for the specific values of %\textcolor{red}
{$A$}, $\sigma_a$, and $\rho$ given in the Table~\ref{tab:dependent_results}

The results in {Table~\ref{tab:dependent_results}} agree with the conclusions in Theorem~\ref{Them-ShapleySuboptimalityGaussian} and 
Theorem~\ref{Them-ShapleySuboptimalityGaussianRandom}, which claim that the two tests $v(i)$ and $\phi(x_i)$ are different in these cases. 
The results in {Table~\ref{tab:dependent_results}}
also follow the main conclusion in Theorem~\ref{Them-ShapleyBetterOrInferior}, which says 
that if $\rho$ is close to but slightly less than 1, 
then $P_{e,\phi} < P_{e,v}$ for {\bf Type A}  attacks/anomalies.
For these same attacks/anomalies, 
Theorem~\ref{Them-ShapleyBetterOrInferior}, also says if $\rho $ is close to but slightly greater than $ -1 $, then $P_{e,\phi} > P_{e,v}$. This is also observed in Table~\ref{tab:dependent_results}.

}

\begin{figure}[htbp]
\centering
\includegraphics[width=0.9\linewidth]{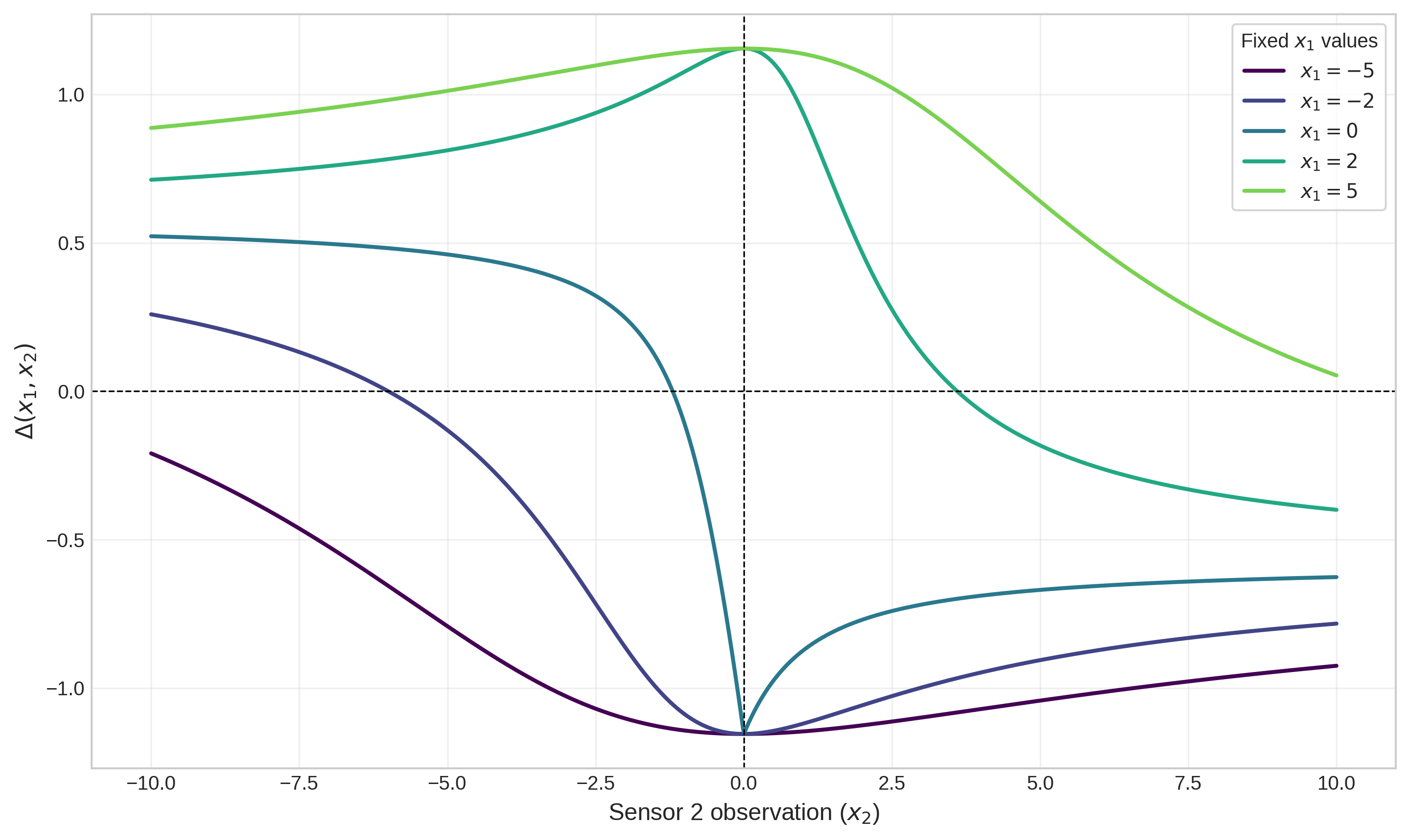}
\caption{
{
Plots of $\Delta(x_1, x_2)$ from Equation~\eqref{delta} as a function of $x_2$ for various fixed values of $x_1$. The observations follow a bivariate Laplacian distribution under an additive attack/anomaly.
}
}
\label{fig:delta_slice}
\end{figure}

\section{Conclusion}

This paper presents a statistical analysis of using the Shapley value for sensor attack/anomaly localization, specifically focusing on using optimum classifiers in the Shapley value calculation. We theoretically and numerically compare the Shapley value test, using $\phi(x_i)$, with the test using $v(i)$, for both independent and dependent observation scenarios. For the case of statistically independent sensor observations, we proved in Theorem III.1 that the Shapley value $\phi(x_i)$ is essentially a scaled version of $v(i)$. Consequently, both tests yield identical decision regions and error probabilities.  Thus, the 
computationally low complexity $v(i)$ test is as good as the 
more computationally expensive Shapley value test in this sense. 

{
In scenarios involving statistically dependent observations, we proved that the Shapley test is not equivalent to the test using $v(i)$ for the some cases involving a bivariate Gaussian and a bivariate Laplacian pdf with constant additive attacks/anomalies. %\textcolor{red}
{We} proved the two tests are also different for some cases with a bivariate Gaussian and Gaussian attacks/anomalies. 
In these cases, the two tests have different decision regions and yield different probabilities of error. We have proven that for some cases with a bivariate Gaussian distribution and a constant positive additive attack/anomaly, the test using Shapley is inferior compared to the test using the $v(i)$ function, while in some other cases the test using Shapley is better.  We also obtained numerical results that support what has been proven to illustrate our findings.

We realize the presented work just takes the first steps towards a full understanding of the statistically characterized performance of using the Shapley value 
for sensor anomaly localization, but we feel that we have made a strong case for the need of further study which we hope the community will help us pursue. 
}

\appendices

\printbibliography
\end{document}